\documentclass[runningheads]{llncs}

\usepackage{amssymb}
\usepackage{amsmath}
\usepackage{booktabs}
\usepackage{multirow}
\usepackage{algorithm}
\usepackage{algorithmic}
\usepackage{bbm}
\usepackage{graphicx}

\begin{document}

\title{Ising-based Consensus Clustering \\ on Specialized Hardware}
%
%
\author{Eldan Cohen\inst{1}\thanks{Work done while at Fujitsu Laboratories of America.} \and
Avradip Mandal\inst{2} \and
Hayato Ushijima-Mwesigwa\inst{2} \and \\
Arnab Roy\inst{2}}
\authorrunning{E. Cohen et al.}
%
\institute{University of Toronto, Toronto, Canada\\
\email{ecohen@mie.utoronto.ca} \and
Fujitsu Laboratories of America, Inc. USA\\
\email{\{amandal,hayato,aroy\}@us.fujitsu.com}}

\maketitle              
\begin{abstract}
The emergence of specialized optimization hardware such as CMOS annealers and adiabatic quantum computers carries the promise of solving hard combinatorial optimization problems more efficiently in hardware. Recent work has focused on formulating different combinatorial optimization problems as Ising models, the core mathematical abstraction used by a large number of these hardware platforms, and evaluating the performance of these models when solved on specialized hardware. An interesting area of application is data mining, where combinatorial optimization problems underlie many core tasks. In this work, we focus on consensus clustering (clustering aggregation), an important combinatorial problem that has received much attention over the last two decades. We present two Ising models for consensus clustering and evaluate them using the Fujitsu Digital Annealer, a quantum-inspired CMOS annealer. Our empirical evaluation shows that our approach outperforms existing techniques and is a promising direction for future research.

\end{abstract}
\setlength{\tabcolsep}{0.25em}

\section{Introduction}
The increasingly challenging task of scaling the traditional Central Processing Unit (CPU) has lead to the exploration of new computational platforms such as quantum computers, CMOS annealers, neuromorphic computers, and so on (see~\cite{coffrin2019evaluating} for a detailed exposition). Although their physical implementations differ significantly, adiabatic quantum computers, CMOS annealers, memristive circuits, and optical parametric oscillators all share Ising models as their core mathematical abstraction \cite{coffrin2019evaluating}. This has lead to a growing interest in the formulation of computational problems as Ising models and in the empirical evaluation of these models on such novel computational platforms. This body of literature includes clustering and community detection \cite{kumar2018quantum,negre2019detecting,shaydulin2019network}, 
graph partitioning \cite{ushijima2017graph,ushijima2019multilevel}, and many NP-Complete problems such as covering, packing, and coloring \cite{lucas2014ising,liu2019modeling}.

Consensus clustering is the problem of combining multiple `base clusterings' of the same set of data points into a single consolidated clustering \cite{ghosh2011cluster}. Consensus clustering is used to generate robust, stable, and more accurate clustering results compared to a single clustering approach \cite{ghosh2011cluster}. The problem of consensus clustering has received significant attention over the last two decades \cite{ghosh2011cluster}, and was previously considered under different names (clustering aggregation, cluster ensembles, clustering combination) \cite{gionis2007clustering}. It has applications in different fields including data mining, pattern recognition, and bioinformatics \cite{gionis2007clustering} and a number of algorithmic approaches have been used to solve this problem.
The consensus clustering is, in essence, a combinatorial optimization problem \cite{wu2014k} and different instances of the problem have been proven to be NP-hard (e.g., \cite{filkov2004integrating,topchy2005clustering}). 

In this work, we investigate the use of special purpose hardware to solve the problem of consensus clustering. To this end, we formulate the problem of consensus clustering using Ising models and evaluate our approach on a specialized CMOS annealer. We make the following contributions:
\begin{enumerate}
    \item We present and study two Ising models for consensus clustering that can be solved on a variety of special purpose hardware platforms.
    \item We demonstrate how our models are embedded on the Fujitsu Digital Annealer (DA), a quantum-inspired specialized CMOS hardware.
    \item We present an empirical evaluation based on seven benchmark datasets and show     our approach outperforms existing techniques for consensus clustering.
\end{enumerate}

\section{Background}
\subsection{Problem Definition}\label{sec:prob_def}
Let $X=\{x_1, ..., x_n\}$ be a set of $n$ data points. A \emph{clustering} of $X$ is a process that partitions $X$ into subsets, referred to as \emph{clusters}, that together cover $X$. A clustering is represented by the mapping $\pi: X \to \{1, \dots, k_{\pi}\}$ where $k_{\pi}$ is the number of clusters produced by clustering $\pi$.  
Given $X$ and a set $\Pi = \{\pi_1, \dots, \pi_m\}$ of $m$ clusterings of the points in $X$, the \emph{Consensus Clustering Problem} is to find a new clustering, $\pi^*$, of the data $X$ that best summarizes the set of clusterings $\Pi$. The new clustering $\pi^*$ is referred to as the \emph{consensus} clustering.

Due to the ambiguity in the definition of an optimal consensus clustering, several approaches have been proposed to measure the solution quality of consensus clustering algorithms \cite{ghosh2011cluster}. In this work, we focus on the approach of determining a consensus clustering that agrees the most with the original clusterings. As an objective measure to determine this agreement, we use the mean Adjusted Rand Index (ARI) metric (Equation \ref{eq:meanARI}). However, we also consider clustering quality measured by mean Silhouette Coefficient \cite{rousseeuw1987silhouettes} and clustering accuracy based on true labels. In Section \ref{sec:emp-evaluation} these evaluation criteria are discussed in more details.

\subsection{Existing Criteria and Methods}\label{sec:existing_approaches}
Various criteria or objectives have been proposed for the Consensus Clustering Problem. In this work we mainly focus on two well-studied criteria, one based on the pairwise similarity of the data points, and the other based on the different assignments of the base clusterings. Other well-known criteria and objectives for the Consensus Clustering Problem can be found in the excellent surveys of  \cite{ghosh2011cluster,vega2011survey}, with most defining NP-Hard optimization problems.

\paragraph{Pairwise Similarity Approaches:} In this approach, a similarity matrix $S$ is constructed such that each entry in $S$ represents the fraction of clusterings in which two data points belong to the same cluster \cite{nguyen2007consensus}. In particular, 
\begin{equation}
    S_{uv} = \frac{1}{m}\sum_{i=1}^m \mathbbm{1}(\pi_i(u) = \pi_i(v)),\label{eq:s_ij}
\end{equation}
with $ \mathbbm{1}$ being the indicator function. The value $S_{uv}$ lies between 0 and 1, and is equal to 1 if all the base clusterings assign points $u$ and $v$ to the same cluster. Once the pairwise similarity matrix is constructed, one can use any similarity-based clustering algorithm on $S$ to find a consensus clustering with a fixed number of clusters, $K$. For example, \cite{li2010combining} proposed to find a consensus clustering $\pi^*$ with exactly $K$ clusters that minimizes the within-cluster dissimilarity: 
\begin{equation}
    \min \sum_{\substack{u, v \in X: \\ \pi^*(u) = \pi^*(v)}} (1 - S_{uv}).\label{eq:sim_based}
\end{equation}

\paragraph{Partition Difference Approaches: }
An alternative formulation is based on the different assignments between clustering. Consider two data points $u, v \in X$, and two clusterings $\pi_i, \pi_j \in \Pi$. The following binary indicator tests if $\pi_i$ and $\pi_j$ disagree on the clustering of $u$ and $v$:
\begin{equation}
    d_{u,v}(\pi_i, \pi_j) = \begin{cases}
    1,& \text{if } \pi_i(u) = \pi_i(v) \text{ and } \pi_j(u) \neq \pi_j(v)\\
    1,& \text{if } \pi_i(u) \neq \pi_i(v) \text{ and } \pi_j(u) = \pi_j(v)\\
    0,& \text{otherwise}.
\end{cases}
\end{equation}
The distance between two clusterings is then defined based on the number of pairwise disagreements:
\begin{equation}
    d(\pi_i, \pi_j) = \frac{1}{2}\sum_{u, v \in X } d_{u,v}(\pi_i, \pi_j)
\end{equation}
with the $\frac{1}{2}$ factor to take care of double counting and can be ignored.
This measure is defined as the number of pairs of points that are in the same cluster in one clustering and in different clusters in the other, essentially considering the (unadjusted) Rand index \cite{ghosh2011cluster}. Given this measure, a common objective is to find a consensus clustering $\pi^*$ with respect to the following optimization problem: 
\begin{equation}
    \min \sum_{i=1}^m d(\pi_i, \pi^*).\label{eq:second_obj}
\end{equation}

\paragraph{Methods and Algorithms:} The two different criteria given above define fundamentally different optimization problems, thus different algorithms have been proposed. One key difference between the two approaches inherently lies in determining the number of clusters $k_{\pi^*}$ in $\pi^*$. The pairwise similarity approaches (e.g., Equation (\ref{eq:sim_based})) require an input parameter $K$ that fixes the number of clusters in $\pi^*$,  whereas the partition difference approaches such as Equation (\ref{eq:second_obj}) do not have this requirement and determining $k_{\pi^*}$ is part of the objective of the problem. Therefore, for example, Equation (\ref{eq:sim_based}) will have a minimum value in the case when $k_{\pi^*}=n$, however this does not hold for Equation (\ref{eq:second_obj}).

The Cluster-based Similarity Partitioning Algorithm (CSPA) is proposed in \cite{strehl2002cluster} for solving the pairwise similarity based approach. The CSPA constructs a similarity-based graph with each edge having a weight proportional to the similarity given by $S$. Determining the consensus clustering with exactly $K$ clusters is treated as a $K$-way graph partitioning problem, which is solved by methods such as METIS  \cite{karypis1998multilevelk}. In \cite{nguyen2007consensus}, the authors experiment with different clustering algorithms including hierarchical agglomerative clustering (HAC) and iterative techniques that start from an initial partition and iteratively reassign points to clusters based on their pairwise similarities.
For the partition difference approach, Li et al. \cite{li2007solving} proposed to solve Equation (\ref{eq:second_obj}) using nonnegative matrix factorization (NMF). Gionis et al. \cite{gionis2007clustering} proposed several algorithms that make use of the connection between Equation (\ref{eq:second_obj}) and the problem of correlation clustering. CSPA, HAC, NMF: these three approaches are considered as baseline in our empirical evaluation section (Section  \ref{sec:emp-evaluation}).

\subsection{Ising Models}
Ising models are graphical models that include a set of nodes representing
spin variables and a set of edges corresponding to the interactions between the spins. The energy level of an Ising model which we aim to minimize is given by: 
\begin{equation}
    E(\sigma) = \sum_{(i,j) \in \mathcal{E}} J_{i,j} \sigma_i\sigma_j + \sum_{i \in \mathcal{N}} h_i \sigma_i, 
 \end{equation}
where the variables $\sigma_i \in \{-1,1\}$ are the spin variables and the couplers, $J_{i,j}$, represent the interaction between the spins.

A Quadratic Unconstrained Binary Optimization (QUBO) model includes binary variables $q_i \in \{0,1\}$ and couplers, $c_{i,j}$. The objective to minimize is: 
\begin{equation}
 E(\textbf{q}) = \sum_{i = 1} ^ n c_iq_i + \sum_{i<j} c_{i,j}q_{i}q_{j}.
 \end{equation} 

QUBO models can be transformed to Ising models by setting $\sigma_i = 2q_i-1$~\cite{bian2010ising}.

\section{Ising Approach for Consensus Clustering on Specialized Hardware}\label{sec:approach}
In this section, we present our approach for solving consensus clustering on specialized hardware using Ising models. We present two Ising models that correspond to the two approaches in Section \ref{sec:existing_approaches}. We then demonstrate how they can be solved on the Fujitsu Digital Annealer (DA), a specialized CMOS hardware. 

\subsection{Pairwise Similarity-based Ising Model}
For each data point $u \in X$, let $q_{uc} \in \{0, 1\}$ be the binary variable such that $q_{uc} = 1$ if $\pi^*$ assigns $u$ to cluster $c$, and 0 otherwise. Then the constraints
\begin{equation}
  \sum_{c=1}^{K}q_{uc} = 1, \quad \text{for each } u \in X
  \label{eq:one_hot}
\end{equation}
ensure $\pi^*$ assigns each point to exactly one cluster. Subject to the constraints (\ref{eq:one_hot}), the sum of quadratic terms  $\sum_{c=1}^{K}  q_{uc} q_{vc} $ is 1 if $\pi^*$ assigns both $u, v \in X$ to the same cluster, and is $0$ if assigned to different clusters. Therefore the value
\begin{equation}
\sum_{\substack{u, v \in X: \\ \pi^*(u) = \pi^*(v)}} (1 - S_{uv}) = \sum_{u, v \in X} (1-S_{uv})
\sum_{c=1}^{K}  q_{uc} q_{vc}
\label{eq:obj}
\end{equation}
represents the sum of within-cluster dissimilarities in $\pi^*$: $(1-S_{uv})$ is the fraction of clusterings in $\Pi$ that assign $u$ and $v$ to different clusters while $\pi^*$ assigns them to the same cluster. We therefore reformulate Equation (\ref{eq:sim_based}) as QUBO:
\begin{equation}
\begin{aligned}
\min \sum_{u, v \in X} (1-S_{uv})
\sum_{c=1}^{K}  q_{uc} q_{vc} + \sum_{u \in X} A (\sum_{c=1}^{K} q_{uc} -1)^2. \label{eq:ising_cluster} 
\end{aligned}
\end{equation}
where the term $ \sum_{u \in X} A (\sum_{c=1}^{K} q_{uc} -1)^2$ is added to the objective function to ensure that the constraints (\ref{eq:one_hot}) are satisfied. $A$ is positive constant that penalizes the objective for violations of constraints (\ref{eq:one_hot}). One can show that if $A \geq n$, the optimal solution of the QUBO in Equation (\ref{eq:ising_cluster}) does not violate the constraints (\ref{eq:one_hot}). The proof is very similar to proof of Theorem \ref{thm:thm1} and a similar result in \cite{kumar2018quantum}.

\subsection{Partition Difference Ising Model}
The partition difference approach essentially considers the (unadjusted) Rand Index \cite{ghosh2011cluster} and therefore can be expected to perform better. The \emph{Correlation Clustering Problem} is another important problem in data mining. Gionis et al. \cite{gionis2007clustering} showed that Equation (\ref{eq:second_obj}) is a restricted case of the Correlation Clustering Problem, and that Equation (\ref{eq:second_obj}) can be expressed as the following equivalent form of the Correlation Clustering Problem
\begin{equation}
    \min_{\pi^*} \ \sum_{\substack{u, v \in X: \\ \pi^*(u) = \pi^*(v)}} (1 - S_{uv}) + \sum_{\substack{u, v \in X: \\ \pi^*(u) \neq \pi^*(v)}} S_{uv}.\label{eq:corr_clustering}
\end{equation}
We take advantage of this equivalence to model Equation  (\ref{eq:second_obj}) as a QUBO. In a similar fashion to the QUBO formulated in the preceding subsection, the terms
\begin{equation}
\sum_{\substack{u, v \in X: \\ \pi^*(u) \neq \pi^*(v)}} S_{uv} =  \sum_{u, v \in X} S_{uv}
\sum_{1 \leq c \neq l \leq K} q_{uc} q_{vl}
\label{eq:obj_corr}
\end{equation}
measure the similarity between points in \emph{different} clusters, where $K$ represents an \textit{upper bound} for the number of clusters in $\pi^*$. This then leads to the minimizing the following QUBO:
\begin{equation}
\begin{aligned}
 \sum_{u, v \in X} (1-S_{uv})
\sum_{c=1}^{K}  q_{uc} q_{vc} + \sum_{u, v \in X} S_{uv}
\sum_{1 \leq c \neq l \leq K} 
q_{uc} q_{vl}+ \sum_{u \in X} B (\sum_{c=1}^{K} q_{uc} -1)^2.
\label{eq:ising_correlation} 
\end{aligned}
\end{equation}

Intuitively, Equation (\ref{eq:ising_correlation}) measures the disagreement between the consensus clustering and the clusterings in $\Pi$. This disagreement is due to points that are clustered together in the consensus clustering but not in the clusterings in $\Pi$, however it is also due to points that are assigned to different clusters in the consensus partition but in the same cluster in some of the partitions in $\Pi$. 

Formally, we can show that Equation (\ref{eq:ising_correlation}) is equivalent to the correlation clustering formulation in Equation (\ref{eq:corr_clustering}) when setting $B \ge n$. Consistent with other methods that optimize Equation (\ref{eq:second_obj}) (e.g., \cite{li2007solving}), our approach takes as an input $K$, an \textit{upper bound} on the number of clusters in $\pi^*$, however the obtained solution can use smaller number of clusters. In our proof, we assume $K$ is large enough to represent the optimal solution, i.e., greater than the number of clusters in optimal solutions to the correlation clustering problem in Equation (\ref{eq:corr_clustering}). 

\begin{theorem}
Let $\bar{\textbf{q}}$ be the optimal solution to the QUBO given by Equation (\ref{eq:ising_correlation}). 
If $B \ge n$, for a  large enough $K \leq n$,  an optimal solution to the Correlation Clustering Problem in Equation~(\ref{eq:corr_clustering}),  $\bar{\pi}$, can be efficiently evaluated from $\bar{\textbf{q}}$. \label{thm:thm1}
\end{theorem}

\begin{proof} 
First we show the optimal solution to the QUBO in Equation (\ref{eq:ising_correlation}) satisfies the one-hot encoding (${\sum_k q_{uk} = 1}$). This would imply given $\bar{\textbf{q}}$ we can create a valid clustering  $\bar{\pi}$.
Note, the optimal solution will never have ${\sum_c q_{uc} > 1}$ as it can only increase the cost. The only case in which an optimal solution will have ${\sum_c q_{uc} < 1}$ is when the cost of assigning a point to a cluster is higher than the cost of not assigning it to a cluster (i.e., the penalty $B$). Assigning a point $u$ to a cluster will incur a cost of $(1 - S_{uv})$ for each point $v$ in the same cluster and $S_{uv}$ for each point $v$ that is not in the cluster. As there is additional $n-1$ points in total, and both $(1 - S_{uv})$ and $S_{uv}$ are less or equal to one (Equation (\ref{eq:s_ij})), setting $B\ge n$ guarantees the optimal solution satisfies the one-hot encoding. 

Now we assume that $\bar{\pi}$ is not optimal, i.e., there exists an optimal solution $\hat{\pi}$ to Equation (\ref{eq:corr_clustering}) that has a strictly lower cost than $\bar{\pi}$. Let $\hat{\textbf{q}}$ be the corresponding QUBO solution to $\hat{\pi}$, such that $\bar{\pi}(u) = k$ if and only if $\bar{q}_{uk} = 1$. This is possible because $K$ is large enough to accomodate all clusters in $\hat{\pi}$. As both $\bar{\textbf{q}}$ and $\hat{\textbf{q}}$ satisfy that one-hot encoding (penalty terms are zero), their cost is identical to the cost of $\bar{\pi}$ and $\hat{\pi}$ . 
Since the cost of $\hat{\pi}$ is strictly lower than $\bar{\pi}$, and the cost of $\bar{\textbf{q}}$ is lower or equal to $\hat{\textbf{q}}$, we have a contradiction. \qed
\end{proof}

\subsection{Solving Consensus Clustering on the Fujitsu Digital Annealer}\label{sec:da}
The Fujitsu Digital Annealer (DA) is a recent CMOS hardware for solving combinatorial optimization problems formulated as QUBO \cite{aramon2019physics,daweb}. 
We use the second generation of the DA that is capable of representing problems with up to 8192 variables with up to 64 bits of precision. The DA has previously been used to solve problems in areas such as communication \cite{naghsh2019digitally} and signal processing \cite{rahman2019ising}. 

The DA algorithm \cite{aramon2019physics} is based on simulated annealing (SA) \cite{kirkpatrick1983optimization}, while taking advantage of the massive parallelization provided by the CMOS hardware \cite{aramon2019physics}. It has several key differences compared to SA, most notably a \textit{parallel-trial} scheme in which each MC step considers all possible one-bit flips in parallel and \textit{dynamic offset} mechanism that increase the energy of a state to escape local minima \cite{aramon2019physics}.

\subsubsection{Encoding Consensus Clustering on the DA}
When embedding our Ising models on the DA, we need to consider the hardware specification and adapt the representation of our model accordingly. Due to hardware precision limit, we need to embed the couplers and biases on an integer scale with limited granularity.
In our experiments, we normalize the pairwise costs $S_{uv}$ in the discrete range $[0, 100]$, $D_{ij} = \left[{S_{uv}\cdot 100}\right]$, and accordingly $(1-S_{uv})$ is replaced by $(100 - D_{uv})$. Note that the theoretical 
bound $B=n$ is adjusted accordingly to be $B=100\cdot n$. 

The theoretical bound guarantees that all constraints are satisfied if problems are solved to optimality. In practice, the DA does not necessarily solve problems to optimality and due to the nature of annealing-based algorithms, using very high weights for constraints is likely to create deep local minima and result in solutions that may satisfy the constraints but are often of low-quality. This is especially relevant to our pairwise similarity model where the bound tends to become loose as the number of clusters grows. In our experiments, we use constant, reasonably high, weights that were empirically found to perform well across datasets. For the pairwise similarity-based model (Equation (\ref{eq:ising_cluster})) we use $A = 2^{14}$, and for the partition difference  model (Equation (\ref{eq:ising_correlation})) we use $B= 2^{15}$. While we expect to get better performance by tuning the weights per-dataset, our goal is to demonstrate the performance of our approach in a general setting. Automatic tuning of the weight values for the DA is a direction for future work.

Unlike many of the existing consensus clustering algorithms that run until convergence, our method runs for a given time limit (defined by the number of runs and iterations) and returns the best solution encountered. In our experiments, we arbitrarily choose \emph{three seconds} as a (reasonably short) time limit to solve our Ising models. As with the weights, we employ a single temperature schedule across all datasets, and \emph{do not} tune it per dataset.

\section{Empirical Evaluation}
\label{sec:emp-evaluation}
We perform an extensive empirical evaluation of our approach using a set of seven benchmark datasets. We first describe how we generate the set of clusterings, $\Pi$. Next, we describe the baselines, the evaluation metrics, and the datasets. 
\subsubsection{Generating Partitions} We follow \cite{fred2005combining} and generate a set of clusterings by randomizing the parameters of the K-Means algorithm, namely the number of clusters $K$ and the initial cluster centers. In this work, we only use labelled datasets for which we know the number of clusters, $\widetilde{K}$, based on the true labels. To generate the base clusterings we run the K-Means algorithm with random cluster centers and we randomly choose $K$ from the range $[2, 3\widetilde{K}]$. For each dataset, we generate 100 clusterings to serve as the clustering set $\Pi$.

\subsubsection{Baseline Algorithms} 
We compare our pairwise similarity-based Ising model, referred to as DA-Sm,  and our correlation clustering Ising model, referred to as DA-Cr, to three popular algorithms for consensus clustering:
\begin{enumerate}
    \item The cluster-based similarity partitioning algorithm (CSPA) \cite{strehl2002cluster} solved as a $K$-way graph partitioning problem using METIS \cite{karypis1998multilevelk}.
    \item The nonnegative matrix factorization (NMF) formulation in \cite{li2007solving}.
    \item Hierarchical agglomerative clustering (HAC) starts with all points in singleton clusters and repeatedly merges the two clusters with the largest average similarity based on $S$, until reaching the desired number of clusters \cite{nguyen2007consensus}.
\end{enumerate}
\subsubsection{Evaluation} We evaluate the different methods using three measures. 
Our main concern in this work is the level of agreement between the consensus clustering and the set of input clusterings. To this end, one requires a metric measuring the similarity of two clusterings that can be used to measure how close the consensus clustering $\pi^*$ to each base clustering $\pi_i \in \Pi$ is. One of popularly used metrics to measure the similarity between two clusterings is the Rand Index (RI) and Adjusted Rand Index (ARI)~\cite{hubert1985comparing}. The Rand Index of two clustering lies between 0 and 1, obtaining the value 1 when both clusterings perfectly agree. Likewise, the maximum score of ARI, which is corrected-for-chance version of RI, is achieved when both clusterings perfectly agree. $ARI(\pi_i, \pi^*)$ can be viewed as measure of \emph{agreement} between the consensus clustering $\pi^*$ and some base clusterings $\pi_i \in \Pi$. We use the mean ARI as the main evaluation criteria:
\begin{equation}
    \label{eq:meanARI}
    \frac{1}{m}\sum_{i=1}^m ARI(\pi_i, \pi^*)
\end{equation}

We also evaluate $\pi^*$ based on clustering quality and accuracy. For clustering quality, we use the  mean Silhouette Coefficient \cite{rousseeuw1987silhouettes} of all data points (computed using the Euclidean distance between the data points). For clustering accuracy, we compute the ARI between the consensus partition $\pi^*$ and the true labels.

\subsubsection{Benchmark Datasets}
We run experiments on seven datasets with different characteristics: \textit{Iris, Optdigits, Pendigits, Seeds, Wine} from the UCI repository~\cite{Dua:2019} as well as \textit{Protein} \cite{xing2003distance} and \textit{MNIST}.\footnote{http://yann.lecun.com/exdb/mnist/} \textit{Optdigits-389} is a randomly sampled subset of Optdigits containing only the digits $\{3,8,9\}$. Similarly, \textit{MNIST-3689} and \textit{Pendigits-149} are subsets of the MNIST and Pendigits datasets.

Table \ref{tab:datasets} provides statistics on each of the data set, with the coefficient of variation (CV) \cite{degroot2012probability} describing the degree of class imbalance: 
zero indicates perfectly balanced classes, while higher values indicate higher degree of class imbalance.

\begin{table}[h!]
\centering
\caption{Datasets}\label{tab:datasets}
\begin{tabular}{lcccc}
\toprule
Dataset & \# Instances & \# Features & \# Clusters & CV \\ \midrule
Iris & 150 & 4 & 3 & 0.000 \\
MNIST-3689 & 389 & 784 & 4 & 0.015 \\
Optdigits-389 & 537 & 64 & 3 & 0.021 \\
Pendigits-149 & 532 & 16 & 3 & 0.059 \\
Protein & 116 & 20 & 6 & 0.301 \\
Seeds & 210 & 7 & 3 & 0.000 \\
Wine & 178 & 13 & 3 & 0.158 \\
\bottomrule
\end{tabular}
\end{table}

\subsection{Results}
We compare the baseline algorithms to the two Ising models in Section \ref{sec:approach} solved using the Fujitsu Digital Annealer described in Section \ref{sec:da}.

Clustering is typically an unsupervised task and the number of clusters is unknown. The number of clusters in the true labels, $\widetilde{K}$, is not available in real scenarios. Furthermore, $\widetilde{K}$ is not necessarily the best value for clustering tasks (e.g., in many cases it is better to have smaller clusters that are more pure). We therefore test the algorithms in two configurations: when the number of clusters is set to  $\widetilde{K}$, as in the true labels, and when the number of clusters is set to $2\widetilde{K}$.

\begin{table}[b]
\centering
\caption{Consensus Performance Measured by Mean ARI Across Partitions}\label{tab:results_consensus}
\begin{tabular}{l|ccccc|ccccc}
\toprule
 & \multicolumn{5}{c|}{$\widetilde{K}$ clusters} & \multicolumn{5}{c}{$2\widetilde{K}$ clusters} \\
Dataset & CSPA & NMF & HAC & DA-Sm & DA-Cr & CSPA & NMF & HAC & DA-Sm & DA-Cr \\ \midrule
Iris & 0.555 & \textbf{0.618} & \textbf{0.618} & \textbf{0.619} & \textbf{0.621} & 0.536 & 0.614 & 0.627 & 0.608 & \textbf{0.642} \\
MNIST & 0.459 & 0.449 & 0.469 & \textbf{0.474} & \textbf{0.474} & 0.456 & 0.511 & 0.517 & 0.490 & \textbf{0.521} \\
Optdig. & 0.528 & \textbf{0.550} & 0.541 & \textbf{0.550} & \textbf{0.551} & 0.492 & 0.596 & 0.608 & 0.576 & \textbf{0.612} \\
Pendig. & 0.546 & 0.546 & 0.507 & \textbf{0.555} & \textbf{0.555} & 0.531 & 0.629 & \textbf{0.642} & 0.605 & \textbf{0.644} \\
Protein & 0.344 & 0.393 & 0.379 & 0.390 & \textbf{0.405} & 0.324 & 0.419 & \textbf{0.423} & 0.378 & 0.415 \\
Seeds & 0.558 & \textbf{0.577} & 0.534 & \textbf{0.575} & \textbf{0.577} & 0.484 & 0.602 & 0.602 & 0.580 & \textbf{0.612} \\
Wine & 0.481 & \textbf{0.536} & 0.535 & \textbf{0.537} & \textbf{0.538} & 0.502 & \textbf{0.641} & \textbf{0.641} & \textbf{0.641} & \textbf{0.643} \\ \midrule
\# Best & 0 & 4 & 1 & 6 & \textbf{7} & 0 & 1 & 3 & 1 & \textbf{6} \\
\bottomrule
\end{tabular}
\end{table}

\subsubsection{Consensus Criteria}
Table \ref{tab:results_consensus} shows the mean ARI between $\pi^*$ and the clusterings in $\Pi$. To avoid bias due to very minor differences, we consider all the methods that achieved Mean ARI that is within a threshold of 0.0025 from the best method to be equivalent and highlight them in bold. We also summarize the number of times each method was considered best across the different datasets.

The results show that DA-Cr is the best performing method for both $\widetilde{K}$ and $2\widetilde{K}$ clusters. The results of DA-Sm are not consistent: DA-Sm and NMF are performing well for $\widetilde{K}$ clusters and HAC is performing better for $2\widetilde{K}$ clusters.

\subsubsection{Clustering Quality} Table \ref{tab:results_silhouette} report the mean Silhouette Coefficient of all data points. Again, DA-Cr is the best performing method across datasets, followed by HAC. NMF seems to be equivalent to HAC for $2\widetilde{K}$.

\begin{table}[h!]
\centering
\caption{Clustering Quality Measured by Silhouette}\label{tab:results_silhouette}
\begin{tabular}{l|ccccc|ccccc}
\toprule
 & \multicolumn{5}{c|}{$\widetilde{K}$ clusters} & \multicolumn{5}{c}{$2\widetilde{K}$ clusters} \\
Dataset & CSPA & NMF & HAC & DA-Sm & DA-Cr & CSPA & NMF & HAC & DA-Sm & DA-Cr \\ \midrule
Iris & 0.519 & \textbf{0.555} & \textbf{0.555} & 0.551 & \textbf{0.553} & 0.289 & 0.366 & \textbf{0.371} & 0.343 & \textbf{0.373} \\
MNIST & 0.075 & 0.072 & \textbf{0.078} & \textbf{0.079} & \textbf{0.078} & 0.069 & \textbf{0.082} & 0.074 & 0.074 & \textbf{0.082} \\
Optdig. & 0.127 & 0.120 & 0.120 & \textbf{0.130} & \textbf{0.130} & 0.088 & \textbf{0.119} & \textbf{0.119} & 0.112 & \textbf{0.121} \\
Pendig. & 0.307 & 0.307 & \textbf{0.315} & 0.310 & 0.310 & 0.305 & 0.332 & \textbf{0.375} & 0.368 & 0.364 \\
Protein & 0.074 & \textbf{0.106} & 0.095 & 0.094 & \textbf{0.104} & 0.068 & 0.111 & 0.115 & \textbf{0.119} & \textbf{0.118} \\
Seeds & 0.461 & 0.468 & 0.410 & 0.469 & \textbf{0.472} & 0.275 & \textbf{0.343} & 0.304 & \textbf{0.344} & 0.302 \\
Wine & 0.453 & 0.542 & \textbf{0.571} & 0.547 & 0.545 & 0.452 & \textbf{0.543} & \textbf{0.541} & 0.539 & \textbf{0.542} \\ \midrule
\# Best & 0 & 2 & 4 & 2 & \textbf{5} & 0 & 4 & 4 & 2 & \textbf{5} \\
\bottomrule
\end{tabular}
\end{table}

\subsubsection{Clustering Accuracy} Table \ref{tab:results_accuracy} shows the clustering accuracy measured by the ARI between $\pi^*$ and the true labels. For $\widetilde{K}$, we find DA-Sm to be best-performing solution (followed by DA-Cr). For $2\widetilde{K}$, DA-Cr outperforms the other methods. Interestingly, there is no clear winner between CSPA, NMF, and HAC.

\begin{table}[h!]
\centering
\caption{Clustering Accuracy Measured by ARI Compared to True Labels}\label{tab:results_accuracy}
\begin{tabular}{l|ccccc|ccccc}
\toprule
 & \multicolumn{5}{c|}{$\widetilde{K}$ clusters} & \multicolumn{5}{c}{$2\widetilde{K}$ clusters} \\
Dataset & CSPA & NMF & HAC & DA-Sm & DA-Cr & CSPA & NMF & HAC & DA-Sm & DA-Cr \\ \midrule
Iris & \textbf{0.868} & 0.746 & 0.746 & 0.716 & 0.730 & 0.438 & 0.463 & 0.447 & 0.433 & \textbf{0.521} \\
MNIST & 0.684 & 0.518 & 0.704 & \textbf{0.730} & 0.720 & 0.412 & 0.484 & \textbf{0.545} & 0.440 & 0.484 \\
Optdig. & 0.712 & 0.642 & 0.675 & 0.734 & \textbf{0.738} & 0.380 & 0.513 & \textbf{0.630} & 0.481 & 0.623 \\
Pendig. & 0.674 & \textbf{0.679} & 0.499 & 0.668 & 0.668 & 0.398 & 0.614 & 0.625 & 0.490 & \textbf{0.639} \\
Protein & 0.365 & 0.298 & 0.363 & 0.349 & \textbf{0.376} & 0.237 & 0.332 & 0.301 & 0.308 & \textbf{0.345} \\
Seeds & 0.705 & 0.710 & 0.704 & \textbf{0.764} & 0.717 & 0.424 & 0.583 & 0.573 & 0.500 & \textbf{0.619} \\
Wine & 0.324 & 0.395 & 0.371 & \textbf{0.402} & 0.398 & 0.231 & 0.245 & 0.240 & \textbf{0.248} & 0.238 \\ \midrule
\# Best & 1 & 1 & 0 & \textbf{3} & 2 & 0 & 0 & 2 & 1 & \textbf{4} \\
\bottomrule
\end{tabular}
\end{table}

\subsubsection{Experiments with higher $K$}\label{sec:larger_k} In partition difference approaches, increasing $K$ does not necessarily lead to a $\pi^*$ that has more clusters. Instead, $K$ serves as an upper bound and new clusters will be used in case they reduce the objective. 

To demonstrate how different algorithms handle different $K$ values, Table \ref{tab:iris_k} shows the consensus criteria and the actual number of clusters in $\pi^*$ for different values of $K$ (note that $\widetilde{K}=3$ in Iris). The results show that the performance of the pairwise similarity methods (CSPA, HAC, DA-Sm) degrades as we increase $K$. This is associated with the fact the actual number of clusters in $\pi^*$ is equal to $K$ which is significantly higher compared to the clusterings in $\Pi$. Methods based on partition difference (NMF and DA-Cr) do not exhibit significant degradation and 
the actual number of clusters does not grow beyond 5 for DA-Cr and 6 for NMF. Note that the average number of clusters in $\Pi$ is $5.26$.

\begin{table}[h!]
\centering
\caption{Results for Iris dataset with different number of clusters}\label{tab:iris_k}
\begin{tabular}{c|ccccc|ccccc}
\toprule
 & \multicolumn{5}{c|}{Consensus Criteria} & \multicolumn{5}{c}{\# of clusters in consensus clustering} \\
$K$ & CSPA & NMF & HAC & DA-Sm & DA-Cr & CSPA & NMF & HAC & DA-Sm & DA-Cr \\ \midrule
$3$ & 0.555 & \textbf{0.618} & \textbf{0.618} & \textbf{0.619} & \textbf{0.621} & 3 & 3 & 3 & 3 & 3 \\
$6$ & 0.536 & 0.614 & 0.627 & 0.608 & \textbf{0.642} & 6 & 6 & 6 & 6 & 5 \\
$9$ & 0.447 & 0.614 & 0.591 & 0.497 & \textbf{0.642} & 9 & 6 & 9 & 9 & 5 \\
$12$ & 0.370 & 0.614 & 0.507 & 0.414 & \textbf{0.642} & 12 & 6 & 12 & 12 & 5 \\
\bottomrule
\end{tabular}
\end{table}

\section{Conclusion}
Motivated by the recent emergence of specialized hardware platforms, we present a new approach to the consensus clustering problem that is based on Ising models and solved on the Fujitsu Digital Annealer, a specialized CMOS hardware. We perform an extensive empirical evaluation and show that our approach outperforms existing methods on a set of seven datasets. These results shows that using specialized hardware in core data mining tasks can be a promising research direction. As future work, we plan to investigate additional problems in data mining that can benefit from the use of specialized optimization hardware as well as experimenting with different types of specialized hardware platforms.

%
%
%


\end{document}